\documentclass[letterpaper, 10 pt, conference]{ieeeconf}  

\IEEEoverridecommandlockouts                              

\usepackage{cite}
\usepackage{amsmath,amssymb,amsfonts}
\usepackage{algorithmic}
\usepackage{graphicx}
\usepackage{textcomp}
\usepackage{subcaption}

\title{\LARGE \bf
On feedforward control using physics--guided neural networks:\\ Training cost regularization and optimized initialization
}

\author{Max Bolderman$^{1}$, Mircea Lazar$^{1}$, and Hans Butler$^{1,2}$
\thanks{\begin{flushright}
\begin{minipage}[r]{0.04\textwidth}
    \includegraphics[width=0.9\linewidth]{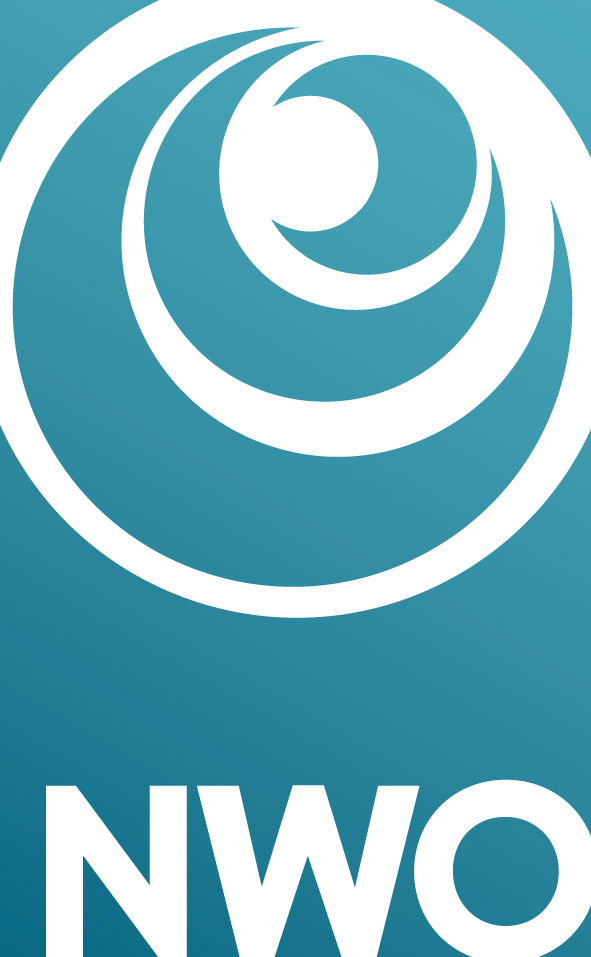}
\end{minipage}
\end{flushright}
\vspace{-1.05cm}
\begin{minipage}[l]{0.43\textwidth}
\hspace{0.5em} *This work is part of the research programme 9654 with project number 17973, which is (partly) financed by the Dutch Research Council (NWO). 
\end{minipage}}
\thanks{$^{1}$Control Systems Group, department of Electrical Engineering, Eindhoven University of Technology, The Netherlands.}%
\thanks{$^{2}$ASML, Veldhoven, The Netherlands.}%
\thanks{Email: {\tt\small m.bolderman@tue.nl, m.lazar@tue.nl, and h.butler@asml.nl} }
}

\begin{document}
\newtheorem{thm}{Definition}[section]
\newtheorem{proposition}[thm]{Proposition}
\newtheorem{lemma}[thm]{Lemma}
\newtheorem{remark}{Remark}[section]
\newtheorem{definition}{Definition}
\newtheorem{assumption}{Assumption}[section]

\maketitle
\thispagestyle{empty}
\pagestyle{empty}

\begin{abstract}
Performance of model--based feedforward controllers is typically limited by the accuracy of the inverse system dynamics model. 
Physics--guided neural networks (PGNN), where a known physical model cooperates in parallel with a neural network, were recently proposed as a method to achieve high accuracy of the identified inverse dynamics.
However, the flexible nature of neural networks can create overparameterization when employed in parallel with a physical model, which results in a parameter drift during training.
This drift may result in parameters of the physical model not corresponding to their physical values, which increases vulnerability of the PGNN to operating conditions not present in the training data. To address this problem, this paper proposes a regularization method via identified physical parameters, in combination with an optimized training initialization that improves training convergence. 
The regularized PGNN framework is validated on a real--life industrial linear motor, where it delivers better tracking accuracy and extrapolation. 
\end{abstract}

\section{Introduction}
Inversion--based feedforward is a control method that improves reference tracking of dynamical systems by generating a control input based on a model of the inverse system dynamics, see, for example, \cite{Boerlage2003, Butterworth2012, Yuen2019, Zundert2018, Igarashi2021}. Typically, linear or linear--in--the parameters physical models are used to parametrize the inverse system model \cite{Steinbuch2000, Yuen2019}. These types of model sets, however, are not able to approximate parasitic nonlinear forces that are always present in mechatronics, which significantly limits the achievable performance \cite{Devasia2002}.

With the aim to approximate and compensate for nonlinear effects, neural networks (NN) have been originally proposed in \cite{Sorensen1999} as inverse system parametrizations. 
NNs are known as universal function approximators \cite{Hornik1989, Cybenko1989}, and are therefore considered capable of identifying the complete inverse dynamics under suitable assumptions. 
Although, there exists some examples of successful implementation of NNs in feedforward control \cite{Otten1997, Ren2009, Li2014}, the lack of physical interpretability and poor extrapolation capabilities \cite{Haley1992} hinders adoption of NN--based controllers in industry. 

To combine the benefits of physical models and neural networks, \cite{Bolderman2021} introduced physics--guided neural networks (PGNN) for feedforward control. The PGNN in \cite{Bolderman2021} employs a parallel model structure that combines the output of a known, not necessarily linear, physics--guided part, with the output of an unknown, possibly nonlinear part that is parameterized by a black--box NN. Both the physics--guided and the NN part are embedded in a single model structure, which allows the PGNN to be identified as a single model. 
The considered model structure is in line with earlier parallel model structures utilized in nonlinear system identification, see, e.g., \cite{Nelles2001, Yasui1996, Zabiri2011}, which typically combine a nonlinear/NN model with a linear model. In \cite{Nelles2001}, two standard approaches for identifying the parameters of such hybrid parallel models are distinguished:
\begin{enumerate}
	\item Simultaneous identification of all parameters;
	\item Sequential (or consecutive) identification: first indentify the parameters of the physics--guided layer and then train the NN using the residuals of the physcs--guided layer. 
\end{enumerate}
In principle, the first approach yields the best data fit and it was also used in \cite{Bolderman2021} to train a feedforward controller. However, this approach can suffer from overaparameterization caused by the flexible nature of the NN that is capable of identifying also parts of the parallel physical model. This creates a parameter drift during training that may yield physically inconsistent parameter values for the physics--guided layer and poor extrapolation capabilities caused by the large NN outputs \cite{Haley1992}. The second approach, used for example in \cite{Yasui1996, Zabiri2011}, preserves the physical model parameters in the physics--guided layer, which can be beneficial for extrapolation outside the training data set. However, approximation accuracy can be reduced because the training has less free parameters to optimize simultaneously. The sequential identification was shown in \cite{Nelles2001} to act as a regularization. 

In this paper we propose the following contributions to the state of the art described above. First, we define a regularized cost function which penalizes, besides the standard data fit, the deviation of the parameters of the physics--guided layer with respect to some known/identified physical parameters. 
The known physical parameters can be obtained based on physical insights or via identification using a linear--in--the--parameters model (LIP). 
Second, we develop a novel initialization procedure for training of parallel PGNNs and we prove that it guarantees an improvement in accuracy of the data fitting with respect to the original LIP model. This is in line with recent initialization procedures for nonlinear system identification based on the best linear approximator, see \cite{Schoukens2020} and the references therein. Experimental results on a coreless linear motor used in lithography industry confirm that the developed regularized training cost function and initialization procedure lead to a high performance and improved robustness to non--training data. 

\begin{remark}
	An alternative regularization method has been proposed in the physics--informed neural network (PINN) literature, see, e.g., \cite{Karpatne2017, Karpatne2019, Karniadakis2019}. 
	These PINNs employ a black--box NN to identify an underlying function based on a cost function that penalizes, besides data fit, compliance of the NN output with the output of an available physical model. 
	The penalty of deviating from the physical model requires the NN to optimize between either data fit or model fit, which results in a biased estimate when the underlying system dynamics deviates from the available physical model. 
	Note that our PGNN design \cite{Bolderman2021} differs by having the physical model as an intrinsic part of the PGNN, whereas the PINNs use the output of physical model only for regularization, i.e. the physical model itself is not embedded in the resulting PINN. 
\end{remark}

The remainder of this paper is organized as follows:
Section~\ref{sec:Preliminaries} introduces the considered system dynamics and PGNN identification setup. The problem statement is formalized in Section~\ref{sec:ProblemStatement}. 
Section~\ref{sec:Regularization} describes the regularization--based PGNN cost function, and Section~\ref{sec:Initialization} explains the PGNN training initialization. 
Afterwards, the developed methodology is tested on an industrial coreless linear motor in Section~\ref{sec:Example}, followed by conclusions in Section~\ref{sec:Conclusion}.

\section{Preliminaries}
\label{sec:Preliminaries}
The symbols $\mathbb{Z}$ and $\mathbb{R}$ denote the set of integers and real numbers, respectively. 
The set $\mathbb{Z}_+ := \{ i \in \mathbb{Z} \mid i > 0 \}$ denotes the positive integers, and the same notation is used for positive real numbers, i.e., $\mathbb{R}_+ := \{i \in \mathbb{R} \mid i > 0 \}$.

\subsection{Inverse system dynamics and modelling}
Consider the following discrete--time inverse dynamical system
\begin{equation}
\label{eq:InverseDynamics}
	u(t) = \theta_0^T T_{\textup{phy}} \big( \phi(t) \big) + g \big( \phi(t) \big), 
\end{equation}
which describes the relation between the input $u(t)$ and the output $y(t)$, that corresponds to a nonlinear ARX model with a particular structure that is clarified as follows.
In~\eqref{eq:InverseDynamics}, $\phi (t) = [y(t+n_a), \hdots, y(t-n_b), u(t-1), \hdots, u(t-n_b) ]^T$ is the regressor, with $n_a, n_b, n_c \in \mathbb{Z}_+$ the orders of the system, and $\theta_0 \in \mathbb{R}^{n_{\theta_0}}$ is the parameter vector corresponding to the known physical model. 
Furthermore, $T_{\textup{phy}} : \mathbb{R}^{n_a+n_b+n_c+1} \rightarrow \mathbb{R}^{n_{\theta_0}}$ comprises of a set of functions that model the known part of the dynamics, e.g., based on physical insights, and $g: \mathbb{R}^{n_a+n_b+1} \rightarrow \mathbb{R}$ is an unknown nonlinear function that contains the \emph{unknown dynamics}. 

Within this paper, the following model structures are considered for identification of the inverse system dynamics~\eqref{eq:InverseDynamics}. 
\begin{definition}
\label{def:LIP}
	A linear--in--the--parameters (LIP) model is defined as
	\begin{equation}
	\label{eq:LIP}
		\hat{u} \big( \theta_{\textup{LIP}}, \phi(t) \big) = \theta_{\textup{LIP}}^T T_{\textup{phy}} \big( \phi(t) \big),
	\end{equation}
	where $\hat{u} \big( \theta_{\textup{LIP}} , \phi(t) \big)$ is the predicted output of the model, and $\theta_{\textup{LIP}} \in \mathbb{R}^{n_{\theta_0}}$ are the parameters to be identified. 
\end{definition}
\begin{definition}
\label{def:PGNN}
	A physics--guided neural network (PGNN) is defined as
	\begin{equation}
	\label{eq:PGNN}
		\hat{u} \big( \theta_{\textup{PGNN}} , \phi(t) \big) = f_{\textup{NN}} \big( \theta_{\textup{NN}}, \phi(t) \big) + \theta_{\textup{phy}}^T T_{\textup{phy}} \big( \phi(t) \big),
	\end{equation}
	where $\theta_{\textup{PGNN}} = \{ \theta_{\textup{NN}}, \theta_{\textup{phy}} \}$ are the PGNN parameters, with $\theta_{\textup{NN}} = \{ W_1, B_1, \hdots , W_{l+1}, B_{l+1} \}$ the neural network weights and biases for each of the $l \in \mathbb{Z}_+$ hidden layers, and $\theta_{\textup{phy}} \in \mathbb{R}^{n_{\theta_0}}$ are the parameters corresponding to the physics--guided part of the PGNN. 
\end{definition}

\begin{remark}
    The NN output is recursively computed as
    \begin{equation}
        \label{eq:NNOutput}
        f_{\textup{NN}} \big( \theta_{\textup{NN}}, \phi(t) \big) = W_{l+1} \alpha_l \big( \hdots \alpha_1 \big( W_1 \phi(t) + B_1 \big)  \big) + B_{l+1},
    \end{equation}
    where $\alpha_i $ contains the activation functions that are applied element--wise.  
\end{remark}

The purpose of the PGNN model~\eqref{eq:PGNN} is to exploit the NN part to identify the inverse dynamics~\eqref{eq:InverseDynamics} more accurately compared to the LIP model that is derived from physical knowledge~\eqref{eq:LIP}. A visualization of the PGNN~\eqref{eq:PGNN}, see also \cite{Bolderman2021}, is shown in Figure~\ref{fig:PGNN}. 

\begin{figure}
	\centering
	\includegraphics[width=1.0\linewidth]{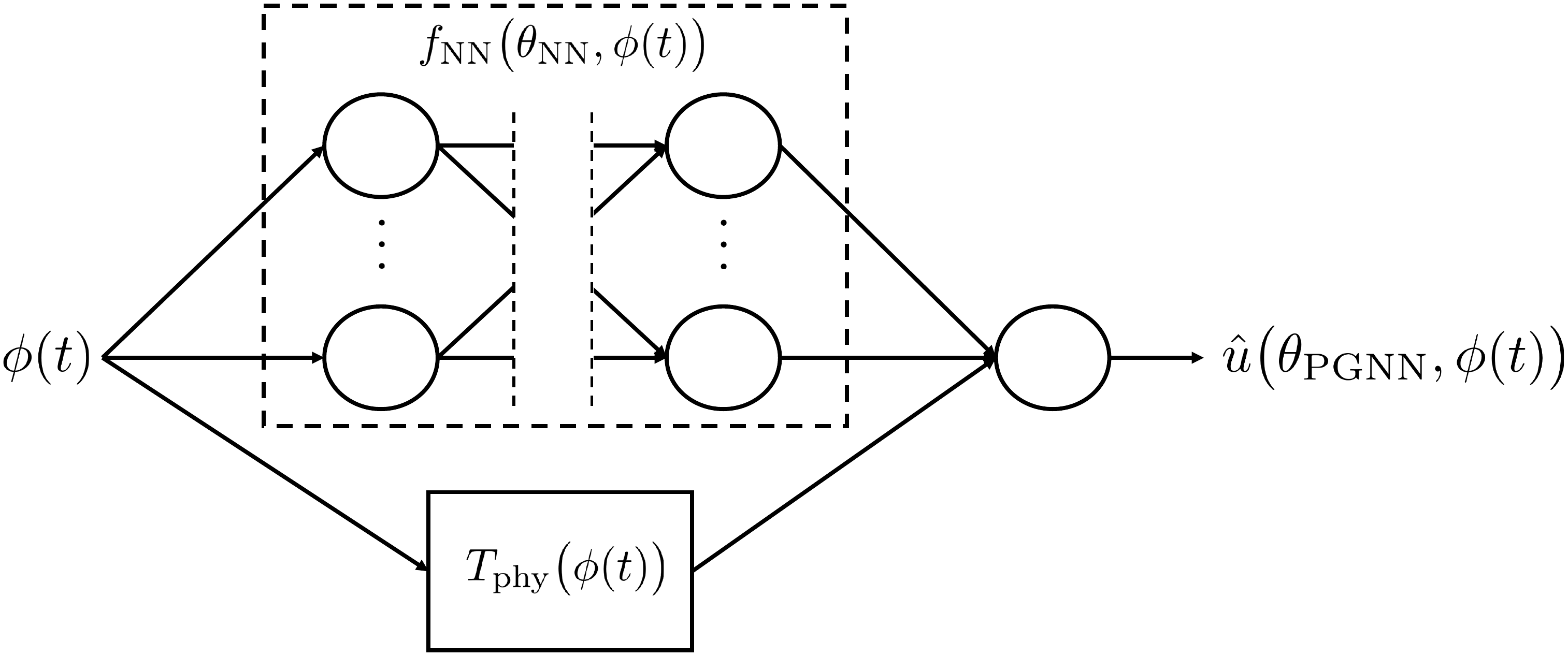}
	\caption{Schematic overview of the considered PGNN structure~\eqref{eq:PGNN}, \cite{Bolderman2021}.}
	\label{fig:PGNN}
\end{figure}

\subsection{Inverse system identification using the LIP model}
Consider that an input--output data set
\begin{equation}
\label{eq:DataSet}
	Z^N = \{ u(0), y(0), \hdots, u(N-1), y(N-1) \},
\end{equation}
is available that is generated by the actual system, i.e., $Z^N$ satisfies the inverse system dynamics~\eqref{eq:InverseDynamics}.

\begin{definition}
\label{def:Uncorrelated}
	Two variables $x_a(t)$ and $x_b(t)$ are uncorrelated if and only if
	\begin{equation}
	\label{eq:Uncorrelated}
		\frac{1}{N} \sum_{t \in Z^N} x_a(t) x_b(t) = 0. 
	\end{equation}
\end{definition}

Let $\theta$ be a general vector of parameters, e.g., $\theta_{\textup{LIP}}$ or $\theta_{\textup{PGNN}}$. 
Then, identification is performed by choosing $\hat{\theta}$ as the minimizing argument of a cost function, i.e.,
\begin{equation}
\label{eq:IdentificationCriterion}
	\hat{\theta} = \textup{arg} \min_{\theta} V \big( \hat{u} \big( \theta, \phi(t) \big), Z^N \big). 
\end{equation}
Typically, the mean--squared error (MSE) 
\begin{equation}
\label{eq:MSE}
	V \big( \hat{u} \big( \theta, \phi(t) \big), Z^N \big) = \frac{1}{N} \sum_{t \in Z^N} \left( u(t) - \hat{u} \big( \theta, \phi(t) \big) \right)^2
\end{equation}
is chosen as a cost function. 
After identification of the parameters $\hat{\theta}$, the corresponding feedforward controller is obtained according to the following definition. 

\begin{definition}
\label{def:FeedforwardController}
    An inversion based feedforward controller is given as
    \begin{equation}
        \label{eq:Feedforward}
        u_{\textup{ff}}(t) = \hat{u} \big( \hat{\theta}, \phi_{\textup{ff}}(t) \big),
    \end{equation}
    with $u_{\textup{ff}}(t)$ the feedforward signal, $\phi_{\textup{ff}} := [r(t+n_a), \hdots, r(t-n_b), u_{\textup{ff}}(t-1), \hdots, u_{\textup{ff}} (t-n_c)]^T$, and $r(t)$ the reference. 
\end{definition}

Substitution of the LIP~\eqref{eq:LIP} in the identification criterion~\eqref{eq:IdentificationCriterion} with MSE cost function~\eqref{eq:MSE}, gives the identified parameters
\begin{align}
\begin{split}
\label{eq:LIPGlobalOptimum}
	\hat{\theta}_{\textup{LIP}} =& M^{-1} \left[ \frac{1}{N} \sum_{t \in Z^N} T_{\textup{phy}} \big( \phi(t) \big) u(t) \right].
\end{split}
\end{align}
where $M := \frac{1}{N} \sum_{t \in Z^N} T_{\textup{phy}} \big( \phi(t) \big) T_{\textup{phy}} \big( \phi(t) \big)^T $. 
The solution~\eqref{eq:LIPGlobalOptimum} is unique if and only if $M$ is non--singular.

\begin{remark}
    In the situation that $T_{\textup{phy}} \big( \phi(t) \big) = \phi(t)$, the identified parameters~\eqref{eq:LIPGlobalOptimum} becomes the best linear approximator (BLA) \cite{Pintelon2012, Schoukens2020}.
\end{remark}

Substitution of the LIP model~\eqref{eq:LIP} with identified parameters~\eqref{eq:LIPGlobalOptimum} into the inverse dynamics~\eqref{eq:InverseDynamics}, gives
\begin{equation}
\label{eq:InverseDynamicsLIP}
	u(t) = \hat{\theta}_{\textup{LIP}}^T T_{\textup{phy}} \big( \phi(t) \big) + f \big( \phi(t) \big), 
\end{equation}
where $f \big( \phi(t) \big) := g \big(\phi(t) \big) + (\theta_0 - \hat{\theta}_{\textup{LIP}} )^T T_{\textup{phy}} \big( \phi(t) \big)$ denotes the \emph{unmodelled dynamics}. 

\begin{lemma}
\label{le:DifferenceUnknownUnmodelledDynamics}
	The \emph{unknown dynamics} $g \big( \phi(t) \big)$ are identical to the \emph{unmodelled dynamics} $f \big( \phi(t) \big)$, if and only if $T_{\textup{phy}} \big( \phi(t) \big)$ and $g \big( \phi(t) \big)$ are uncorrelated. 
\end{lemma}
\begin{proof}
	The \emph{unknown} and \emph{unmodelled} dynamics are identical if and only if $\theta_0 - \hat{\theta}_{\textup{LIP}} =0$. 
	Using the solution $\hat{\theta}_{\textup{LIP}}$~\eqref{eq:LIPGlobalOptimum} and the inverse dynamics~\eqref{eq:InverseDynamics} gives
	\begin{align}
	\begin{split}
	\label{eq:Proof1Step1}
		\theta_0 - \hat{\theta}_{\textup{LIP}} & = M^{-1} M \theta_0 - M^{-1}\left[ \frac{1}{N} \sum_{t \in Z^N} T_{\textup{phy}} \big( \phi(t) \big) u(t) \right] \\
		& = M^{-1} \left[ \frac{1}{N} \sum_{t \in Z^N} T_{\textup{phy}} \big( \phi(t) \big) g \big( \phi(t) \big) \right],
	\end{split}
	\end{align}
	which proves Lemma~\ref{le:DifferenceUnknownUnmodelledDynamics}, since $M$ is non--singular. 
\end{proof}

\section{Problem statement}
\label{sec:ProblemStatement}
\subsection{Inverse system identification using the PGNN model}
Identification of the inverse system dynamics~\eqref{eq:InverseDynamics} using the LIP model~\eqref{eq:LIP} leaves room for improvement, as there is still the unmodelled dynamics $f \big( \phi(t) \big)$ to be identified. 
Therefore, it was proposed in \cite{Bolderman2021} to use the PGNN model structure~\eqref{eq:PGNN} for identification of the inverse system dynamics. 
Due to the parallel layer structure of the PGNN, if the cost function~\eqref{eq:MSE} is used for identification, there is no guarantee that the resulting parameters for the physical layer correspond to the parameters of the identified LIP model~\eqref{eq:LIP} of the original system~\eqref{eq:InverseDynamics}. In fact, the NN part of the PGNN~\eqref{eq:PGNN} can start to compete with the physical model. 

Minimizing the cost~\eqref{eq:MSE} solely aims to optimize the data fit of the PGNN and does not attribute specific dynamics to the neural layer versus the physics--guided layer. 
It would be desirable to steer the training of the PGNN such that the physics--guided layer identifies the known physical dynamics and the NN layer identifies the remaining unmodelled dynamics.
This is a well known problem in nonlinear system identification when models with a parallel structure are used, see \cite{Nelles2001}, Chapter~$21$. 
Therein, a solution is proposed which is based on first identifying the known physical dynamics, and then training the NN using the residuals. 
This preserves the known dynamics in the physical part of the PGNN~\eqref{eq:PGNN}, but also limits achievable data fitting accuracy, as the whole set of parameters for the PGNN are not identified simultaneously.

\subsection{Illustrative example of layer competition}
In order to illustrate the competition between the NN and physics--guided layer in the PGNN~\eqref{eq:PGNN}, we consider the coreless linear motor (CLM) also used in \cite{Bolderman2021}. 
A reasonably accurate LIP model can be obtained via first--principle modelling using Newton's second law, which gives
\begin{equation}
\label{eq:CLMDynamics}
	u(t) = \begin{bmatrix} m & f_v & f_c & f_k \end{bmatrix} \begin{bmatrix} \delta^2 y(t) \\ \delta y(t) \\ \textup{sign} \big( \delta y(t) \big) \\ y(t) \end{bmatrix} + g \big( \phi(t) \big),
\end{equation}
where $\delta = \frac{1-q^{-1}}{T_s}$ is the backward Euler differentiation, with $q^{-1}$ the backwards--shift operator, and $T_s$ the sampling time. 
Also, $\phi(t) = [ y(t), y(t-1), y(t-2) ]^T$, and $m, f_v, f_c, f_k \in \mathbb{R}_+$ are the mass, viscous friction coefficient, Coulomb friction coefficient, and stiffness coefficient, respectively.

\begin{figure}
	\begin{subfigure}{1\linewidth}
	\centering
	\includegraphics[width=1\linewidth]{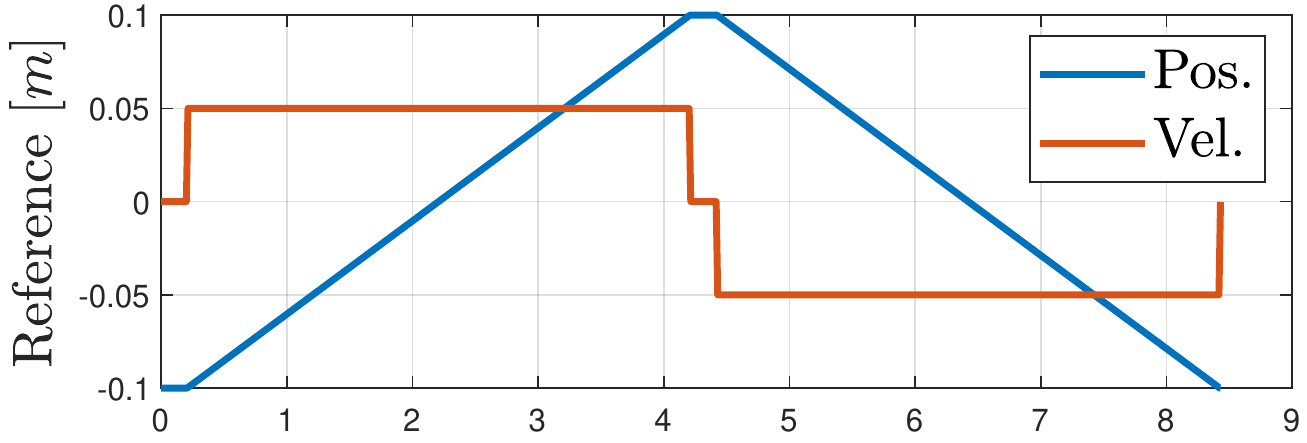}
	\end{subfigure}\hfill
	\begin{subfigure}{1\linewidth}
	\centering
	\includegraphics[width=1\linewidth]{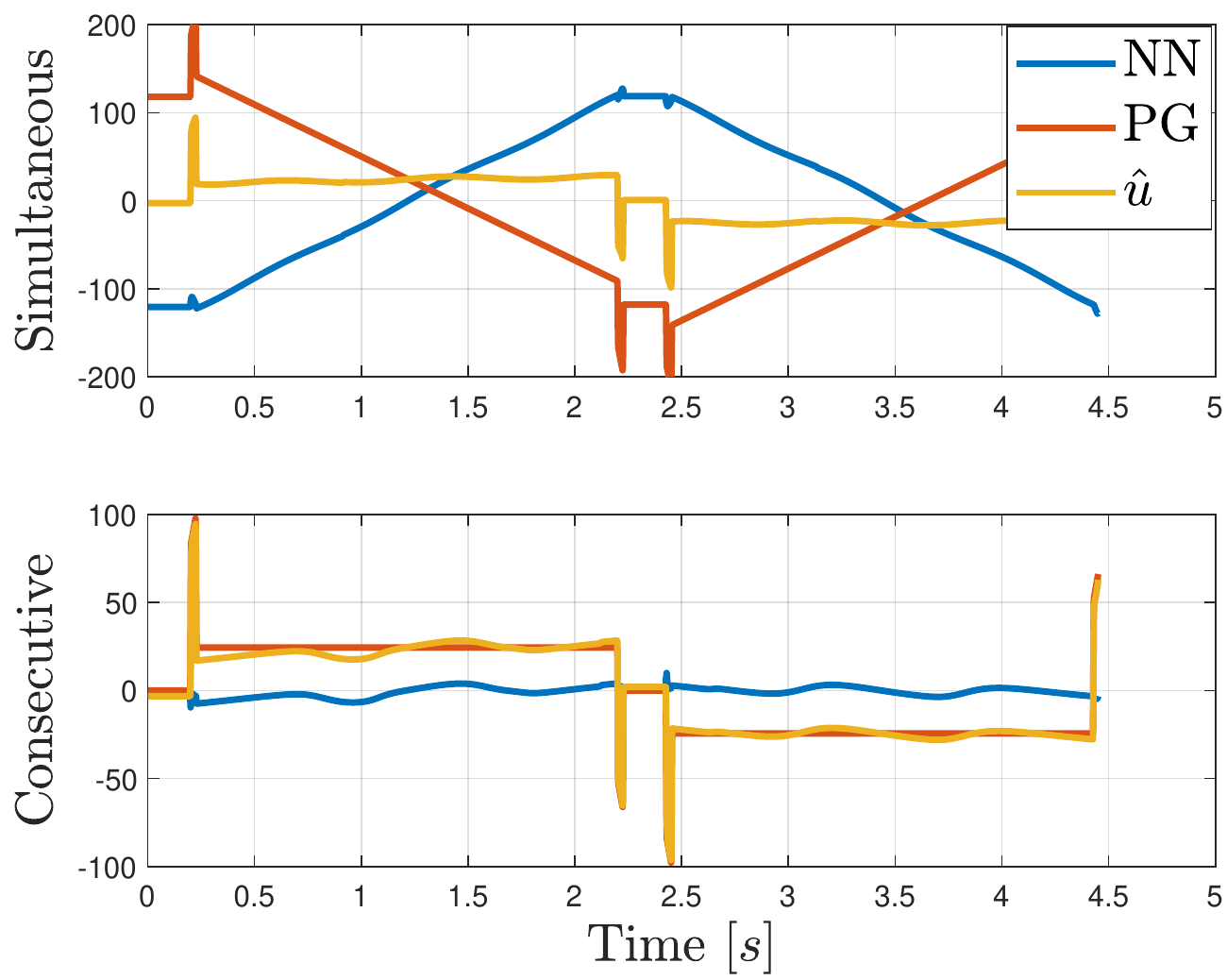}
	\end{subfigure}
	\caption{Generated feedforward signals resulting from the reference (top), for a PGNN~\eqref{eq:PGNN} where the NN and physics--guided (PG) layer are trained either simultaneously (middle), or sequentially (bottom). }
	\label{fig:ExampleFeedforward}
\end{figure}

Identification using only the known, LIP model in~\eqref{eq:CLMDynamics} gives $\hat{\theta}_{\textup{LIP}} = [18.8, 172, 7.21, 1.36\cdot 10^{-8}]^T$, whereas the corresponding parameters in the PGNN~\eqref{eq:CLMDynamics} converged to $\hat{\theta}_{\textup{phy}} = [18.2, 301, -4.99, 1.18\cdot 10^3 ]^T$. 
The results are obtained using a single hidden layer NN with $16$ neurons that have $\tanh (\cdot)$ activation functions. 
This parameter drift during training has two main disadvantages.

First of all, the parameter drift results in an inconsistent physics--guided layer in the PGNN~\eqref{eq:PGNN}. 
This complicates validation of the PGNN and thereby hinders applicability to safety critical systems.
Secondly, extrapolation capabilities of the PGNN to operating conditions not present in the training data are severely limited. 
This is caused by the large output of the NN layer, and uninterpretable physics--guided layer, as is shown in Figure~\ref{fig:ExampleFeedforward}.
Although the NN has universal approximation capabilities to identify the dynamics in the domain for which training data is generated, it is well known that its extrapolation capabilites are limited \cite{Haley1992}. 
Therefore, it is desired to have the NN layer output small, while the parameters of the physics--guided layer of the PGNN remain close to the physical parameters.

\subsection{Problem formulation}
In order to solve the aforementioned problems, in this paper we propose a regularization--based approach for simultaneously training the PGNN~\eqref{eq:PGNN}. 
Under the condition that PGNN training converges to a global optimum, the regularized training allow for the following:
\begin{enumerate}
	\item When the NN layer can identify the \emph{unmodelled dynamics}, the physics--guided layer must become the LIP, and the NN must identify the \emph{unmodelled dynamics}. 
	\item When the NN cannot identify the \emph{unmodelled dynamics}, the physics--guided layer parameters can still be partially used to improve data--fit. 
\end{enumerate}

Since training of the PGNN~\eqref{eq:PGNN} is a non--convex optimization, it is possible that the optimization scheme ends up in a local minimum. 
Therefore, the second problem considered in this paper is defined as the design of a PGNN parameter initialization method that returns a smaller cost function compared to the LIP model. This property must also be satisfied for the regularized PGNN cost function.

\section{Regularized PGNN training}
In order to have the physical parameters $\theta_{\textup{phy}}$ of the PGNN~\eqref{eq:PGNN} remain close to the LIP parameters $\hat{\theta}_{\textup{LIP}}$~\eqref{eq:LIPGlobalOptimum}, the regularized PGNN cost function is defined. 
\label{sec:Regularization}
\begin{definition}
	The regularized PGNN training cost function is defined as
	\begin{align}
	\begin{split}
	\label{eq:RegularizedCostFunction}
		V & \big( \hat{u} \big( \theta_{\textup{PGNN}}, \phi(t) \big), Z^N \big) = \frac{1}{N} \sum_{t \in Z^N} \big(u(t) - \hat{u} \big( \theta_{\textup{PGNN}}, \phi(t) \big) \big)^2 \\
		& \quad \quad + ( \theta_{\textup{phy}} - \hat{\theta}_{\textup{LIP}} )^T \Lambda (\theta_{\textup{phy}} - \hat{\theta}_{\textup{LIP}}), 
	\end{split}
	\end{align}
	where $\Lambda$ is a symmetric positive definite matrix that defines the relative weights for penalizing deviation from the LIP. 
\end{definition}
\begin{remark}
	The regularized PGNN training cost function~\eqref{eq:RegularizedCostFunction} is fundamentally different from the PINNs regularization
	\begin{align}
	\begin{split}
	\label{eq:PINNsRegularization}
		V  & \big( \hat{u} \big( \theta_{\textup{NN}}, \phi(t) \big), Z^N \big) = \frac{1}{N} \sum_{t \in Z^N} \big(u(t) - \hat{u} \big( \theta_{\textup{NN}}, \phi(t) \big) \big)^2 \\
		& + \lambda \frac{1}{N} \sum_{t \in Z^N} \left( \hat{u} \big( \theta_{\textup{NN}} , \phi(t) \big) - \theta_{\textup{phy}}^T T_{\textup{phy}} \big( \phi(t) \big) \right)^2 ,
	\end{split}
	\end{align}
	with $\lambda \in \mathbb{R}_+$ the relative weight of the regularization. 
	Training of the PINNs can be done either by optimizing only $\theta_{\textup{NN}}$ and fixing $\theta_{\textup{phy}} = \hat{\theta}_{\textup{LIP}}$, or by simultaneously training $\{\theta_{\textup{NN}}, \theta_{\textup{phy}}\}$. 
	Clearly, it is not possibly to have both contributions in~\eqref{eq:PINNsRegularization} become zero when there is unknown dynamics, see $g\big( \phi(t) \big)$ in~\eqref{eq:InverseDynamics}. 
	This implies that the PINNs approach \cite{Karpatne2017, Karpatne2019, Karniadakis2019} must optimize between either data fit or physical model compliance. 
	Straightforwardly, this creates a bias in the estimation of $\hat{\theta}_{\textup{NN}}$. 
\end{remark}

It can be proven that the PGNN regularization~\eqref{eq:RegularizedCostFunction} does not create a bias under the following assumptions.

\begin{assumption}
\label{as:SufficientNNPGNN}
	The NN in the PGNN~\eqref{eq:PGNN} is chosen sufficiently rich to identify the \emph{unmodelled dynamics}, i.e., there exists a $\theta_{\textup{NN}}^*$ such that
	\begin{equation}
	\label{eq:SufficientNNPGNN}
		f_{\textup{NN}} \big( \theta_{\textup{NN}}^*, \phi(t) \big) = f \big( \phi(t) \big). 
	\end{equation}
	This $\theta_{\textup{NN}}^*$ is unique up until the level of interchanging neurons in the same layer. 
\end{assumption}
\begin{assumption}
\label{as:PersistenceOfExcitationNNPGNN}
	The data set $Z^N$~\eqref{eq:DataSet} is persistently exciting with respect to the NN part of the PGNN~\eqref{eq:PGNN}, such that
	\begin{equation}
	\label{eq:PersistenceOfExcitationNNPGNN}
		\frac{1}{N} \sum_{t \in Z^N} \left( f_{\textup{NN}} \big( \theta_{\textup{NN}}, \phi(t) \big) - f \big( \phi(t) \big) \right)^2 = 0 \; \Rightarrow \; \theta_{\textup{NN}} = \theta_{\textup{NN}}^*.
	\end{equation}
\end{assumption}

\begin{proposition}
\label{prop:PGNNFitting}
	Consider the data generating system~\eqref{eq:InverseDynamics}, the PGNN~\eqref{eq:PGNN}, and the identification criterion~\eqref{eq:IdentificationCriterion} with regularized PGNN cost function~\eqref{eq:RegularizedCostFunction}. 
	Suppose that Assumptions~\ref{as:SufficientNNPGNN} and~\ref{as:PersistenceOfExcitationNNPGNN} hold, and let $\hat{\theta}_{\textup{LIP}}$ be calculated as in~\eqref{eq:LIPGlobalOptimum}. 
	Assume that the minimization of the cost function~\eqref{eq:RegularizedCostFunction} converges to a global optimum. 
	Then, the corresponding optimal parameters satisfy $\hat{\theta}_{\textup{PGNN}} = \{ \hat{\theta}_{\textup{LIP}}, \theta_{\textup{NN}}^* \}$.
\end{proposition}

\begin{proof}
	Writing out the regularized cost function~\eqref{eq:RegularizedCostFunction} and substituting the PGNN~\eqref{eq:PGNN} gives
	\begin{align}
	\begin{split}
	\label{eq:Proof2Step1}
		&V \big( \hat{u} \big( \theta_{\textup{PGNN}}, \phi(t)\big), Z^N \big) = \\
		&\frac{1}{N} \sum_{t \in Z^N} \left( u(t) - \theta_{\textup{phy}}^T T_{\textup{phy}} \big( \phi(t) \big) - f_{\textup{NN}} \big( \theta_{\textup{NN}}, \phi(t) \big) \right)^2 + \\
		& \quad \quad ( \theta_{\textup{phy}}^T - \hat{\theta}_{\textup{LIP}} )^T \Lambda ( \theta_{\textup{phy}} - \hat{\theta}_{\textup{LIP}} ) \geq 0.  
	\end{split}
	\end{align}
	Both terms in~\eqref{eq:Proof2Step1} are non--negative. Correspondingly, the global optimum is attained if $\hat{\theta}_{\textup{phy}} = \hat{\theta}_{\textup{LIP}}$ (regularization term), and $\hat{\theta}_{\textup{NN}} = \theta_{\textup{NN}}^*$ (data--fit term, substitute $u(t)$ from~\eqref{eq:InverseDynamics}). 
\end{proof}

\begin{remark}
	It is in general difficult to know the required NN dimensions (number of hidden layers, number of neurons per hidden layer) for Assumption~\ref{as:SufficientNNPGNN} to be satisfied. 
	In the situation that Assumption~\ref{as:SufficientNNPGNN} is violated, the training of the PGNN~\eqref{eq:PGNN} with regularized cost function~\eqref{eq:RegularizedCostFunction} allows the physics--guided layer parameters to slightly deviate from the LIP in order to better fit the data. 
\end{remark}

\section{PGNN training initialization}
\label{sec:Initialization}
Proposition~\ref{prop:PGNNFitting} holds under the assumption that training of the PGNN~\eqref{eq:PGNN} converges to the global minimum of the cost function~\eqref{eq:RegularizedCostFunction}. 
The non--convexity of the optimization however, requires the use of nonlinear optimization schemes that tend to get stuck in local minima, which strongly depends on the initialized parameters \cite{Nocedal2006}. 
For this reason, there is a strong interest in finding initial parameter values that already perform well on the cost function~\eqref{eq:RegularizedCostFunction}. 

In order to illustrate the proposed PGNN parameter initialization, we denote $\theta_{\textup{PGNN}}^{(k)}$ as the value of $\theta_{\textup{PGNN}}$ at the $k$'th iteration in training, and rewrite the PGNN output~\eqref{eq:PGNN} as
\begin{align}
\begin{split}
\label{eq:PGNNOutputRewritten}
	\hat{u} \big( \theta_{\textup{PGNN}} , \phi(t) \big) = & W_{l+1}^T f_{\textup{HL}} \big( \theta_{\textup{HL}}, \phi(t) \big) \\
	& + B_{l+1} + \theta_{\textup{phy}}^T T_{\textup{phy}} \big( \phi(t) \big),
\end{split}
\end{align}
where $\theta_{\textup{HL}} := \{ W_1, B_1, \hdots, W_l, B_l \}$, such that $\theta_{\textup{NN}} = \{ \theta_{\textup{HL}}, W_{l+1}, B_{l+1} \}$.
Also, $f_{\textup{HL}} : \mathbb{R}^{n_a+ n_b + n_c + 1} \rightarrow \mathbb{R}^{n_l}$ denotes the output of the last hidden layer $l \in \mathbb{Z}_+$ that contains $n_l\in \mathbb{Z}_+$ neurons.
In contrary to the majority of NN literature that initializes all the weights and biases $\theta_{\textup{NN}}^{(0)}$ randomly, we only do so for the hidden layer weights and biases $\theta_{\textup{HL}}^{(0)}$, and define $\phi_{\textup{OL}}^{(0)}(t) := [ f_{\textup{HL}} \big( \theta_{\textup{HL}}^{(0)}, \phi(t) \big)^T, 1, T_{\textup{phy}} \big( \phi(t) \big)^T]^T$. 
Then, under the following assumption, we prove that the PGNN can be initialized at a lower regularized cost function~\eqref{eq:RegularizedCostFunction} compared to the LIP model~\eqref{eq:LIP}. 

\begin{assumption}
\label{as:PerstenceOfExcitationOutputParameters}
	The data set $Z^N$~\eqref{eq:DataSet} is persistently exciting for the output weights and biases of the PGNN~\eqref{eq:PGNN}, i.e.,
	\begin{equation}
	\label{eq:PersistenceOfExcitationOutputParameters}
		M_{\textup{R}} : = \frac{1}{N} \sum_{t \in Z^N} \phi_{\textup{OL}}^{(0)} (t) {\phi_{\textup{OL}}^{(0)}(t)}^T + \begin{bmatrix} 0 \\ 0 \\ I \end{bmatrix} \Lambda \begin{bmatrix} 0 \\ 0 \\ I \end{bmatrix}^T
	\end{equation}
	is non--singular. 
\end{assumption}

\begin{lemma}
\label{le:LowerCostFunction}
	Consider the PGNN~\eqref{eq:PGNN} with parameters $\theta_{\textup{HL}}^{(0)}$ initialized randomly, and the regularized cost function~\eqref{eq:RegularizedCostFunction}. 
	Suppose that Assumption~\ref{as:PerstenceOfExcitationOutputParameters} holds. 
	Define $\theta_{\textup{OL}} := [W_{l+1}^T , B_{l+1}, \theta_{\textup{phy}}^T ]^T$ and initialize as
	\begin{equation}
	\label{eq:PGNNInitializationParameters}
		\theta_{\textup{OL}}^{(0)} = M_{\textup{R}}^{-1} \left[ \frac{1}{N} \sum_{t \in Z^N} u(t) \phi_{\textup{OL}}^{(0)} (t) + \begin{bmatrix} 0 \\ 0 \\ I \end{bmatrix} \Lambda \hat{\theta}_{\textup{LIP}} \right]. 
	\end{equation}
	Then, it holds that
	\begin{equation}
	\label{eq:LowerCostFunction}
		V \left( \hat{u} \big( \theta_{\textup{PGNN}}^{(0)}, \phi(t) \big), Z^N \right) \leq V \left( \hat{u} \big( \hat{\theta}_{\textup{LIP}}, \phi(t) \big), Z^N \right) ,
	\end{equation}
	with strict inequality if and only if
	\begin{align}
	\begin{split}
	\label{eq:LowerCostFunctionCondition}
		M_{\textup{R}} \begin{bmatrix} 0 \\ 0 \\ \hat{\theta}_{\textup{LIP}} \end{bmatrix} - \left( \frac{1}{N} \sum_{t \in Z^N} u(t) \phi_{\textup{OL}}^{(0)} (t) + \hat{\theta}_{\textup{LIP}}^T \Lambda \begin{bmatrix} 0 \\ 0 \\ I \end{bmatrix}^T \right) \neq 0.
	\end{split}
	\end{align}
\end{lemma}

\begin{proof}
	Rewriting the regularized cost function~\eqref{eq:RegularizedCostFunction} gives
	\begin{align}
	\begin{split}
	\label{eq:Proof3Step1}
		V &\big( \theta_{\textup{OL}}^T \phi_{\textup{OL}}^{(0)} , Z^N \big) = \frac{1}{N} \sum_{t \in Z^N} u(t)^2 + \hat{\theta}_{\textup{LIP}}^T \Lambda \hat{\theta}_{\textup{LIP}} + \theta_{\textup{OL}}^T M_{\textup{R}} \theta_{\textup{OL}}  \\
		& - 2 \left( \frac{1}{N} \sum_{t \in Z^N} u(t) {\phi_{\textup{OL}}^{(0)}}^T + \hat{\theta}_{\textup{LIP}}^T \Lambda \begin{bmatrix} 0 & 0 & I \end{bmatrix} \right) \theta_{\textup{OL}},
	\end{split}
	\end{align}
	that has the global optimum $\theta_{\textup{OL}}^{(0)}$ in~\eqref{eq:PGNNInitializationParameters}, which is unique due to Assumption~\ref{as:PerstenceOfExcitationOutputParameters}. 
	Correspondingly, we can lowerbound the regularized cost function~\eqref{eq:RegularizedCostFunction} according to
	\begin{align}
	\begin{split}
	\label{eq:Proof3Step2}
		V &  \left( \hat{u} \big( \theta_{\textup{PGNN}}^{(0)}, \phi(t) \big) , Z^N \right) = V \left( {\theta_{\textup{OL}}^{(0)}}^T \phi_{\textup{OL}}^{(0)} (t) , Z^N \right) \\
		& \quad \leq V \left( \begin{bmatrix} 0 & 0 & \hat{\theta}_{\textup{LIP}} \end{bmatrix}^T \phi_{\textup{OL}}^{(0)}(t), Z^N \right) \\
		&  \quad = V \left( \hat{u} \big( \hat{\theta}_{\textup{LIP}}, \phi(t) \big) , Z^N \right).
	\end{split}
	\end{align}
	In~\eqref{eq:Proof3Step2}, the inequality holds with equality if and only if $\frac{\partial V \big( \theta_{\textup{OL}}^T \phi_{\textup{OL}}^{(0)} (t) , Z^N \big)}{\partial \theta_{\textup{OL}} } \big|_{\theta_{\textup{OL}} = [0, 0, \hat{\theta}_{\textup{LIP}}^T ]^T} = 0$. 
	If~\eqref{eq:LowerCostFunctionCondition} holds, this equality is not satisfied and we have~\eqref{eq:LowerCostFunction} with strict inequality. 
\end{proof}

After initialization of the parameters $\theta_{\textup{PGNN}}^{(0)} = \{ \theta_{\textup{OL}}^{(0)}, \theta_{\textup{HL}}^{(0)} \}$ as in Lemma~\ref{le:LowerCostFunction}, training $\theta_{\textup{PGNN}}$ is performed using a nonlinear optimization scheme that returns $\{\theta_{\textup{PGNN}}^{(0)}, \hdots, \theta_{\textup{PGNN}}^{(k)} \}$, with $k$ the number of iterations for the solver to converge or stop. 
Then, choosing $\hat{\theta}_{\textup{PGNN}}$ as the iteration with the smallest cost function, concludes that
\begin{align}
\begin{split}
\label{eq:LowerCostFunctionFinal}
	V \big( \hat{u} \big( \hat{\theta}_{\textup{PGNN}}, \phi(t) \big), Z^N \big) & \leq V \big( \hat{u} \big( \theta_{\textup{PGNN}}^{(0)}, \phi(t) \big), Z^N \big) \\
	& < V \big( \hat{u} \big( \hat{\theta}_{\textup{LIP}}, \phi(t) \big), Z^N \big),
\end{split}
\end{align}
if condition~\eqref{eq:LowerCostFunctionCondition} holds. 

\begin{remark}
	Beneficial of the proposed PGNN~\eqref{eq:PGNN} is that Lemma~\ref{le:LowerCostFunction} does not depend on the dimensions of the NN part of the PGNN. 
	This is in contrast to the PINNs approach, where there are no clear guidelines for the design of the NN.
\end{remark}

\begin{figure}
	\begin{subfigure}{1\linewidth}
	\centering
	\includegraphics[width=1\linewidth]{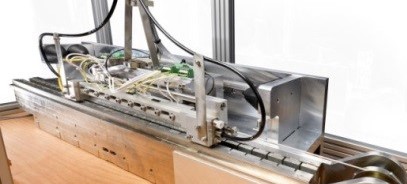}
	\end{subfigure}\hfill
	\begin{subfigure}{1\linewidth}
	\centering
	\includegraphics[width=1\linewidth]{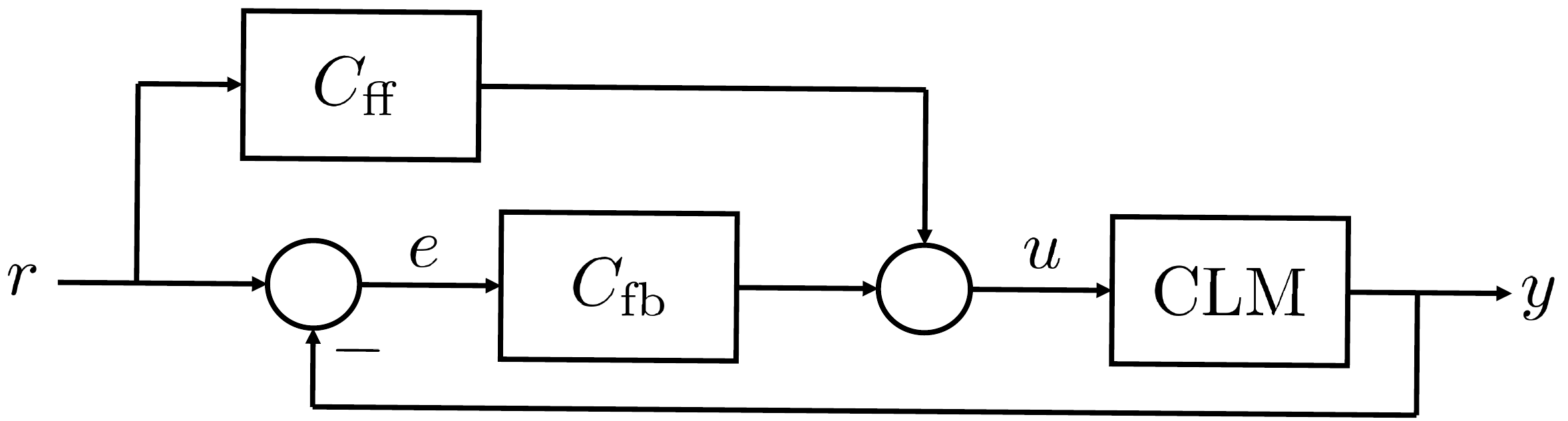}
	\end{subfigure}
	\caption{Experimental CLM setup, and a schematic overview of the closed--loop control structure.}
	\label{fig:CLM}
\end{figure}

\section{Experimental validation}
\label{sec:Example}
To illustrate the effectiveness of the regularized PGNN~\eqref{eq:PGNN} for feedforward control as defined in \eqref{eq:Feedforward}, we consider the problem of position control for a real--life coreless linear motor (CLM) shown in Figure~\ref{fig:CLM}.
The system is controlled in closed--loop by a feedback controller proposed in~\cite{Bolderman2021}. 
Training data is generated while operating the CLM in closed--loop at a frequency of $10$ $kHz$, with a third order nominal reference $r_1(t) := \{ r(0), \hdots, r(N_R-1) \}$ that moves back--and--forth in $r_1(t) \in \{ -0.1, 0.1\}$ $m$, with maximum velocity $\max(| \dot{r}_1(t) |) = 0.05$ $\frac{m}{s}$, acceleration $\max(| \ddot{r}_1(t) |) = 4$ $\frac{m}{s^2}$, and jerk $\max(| \dddot{r}_1(t) |) = 1000$ $\frac{m}{s^3}$. Additionally, the input to the system $u(t)$ is dithered with a normally distributed white noise $\Delta u(t) \sim \mathcal{N} (0, 50^2 )$.

Following the previously introduced CLM dynamics~\eqref{eq:CLMDynamics}, the PGNN is defined as
\begin{equation}
\label{eq:PGNNsCLM}
	\hat{u} \big( \theta_{\textup{PGNN}}, \phi_{\textup{CLM}}(t) \big) = \theta_{\textup{phy}}^T \phi_{\textup{CLM}}(t) + f_{\textup{NN}} \big(\theta_{\textup{NN}}, \phi_{\textup{CLM}}(t) \big),
\end{equation}
with $\phi_{\textup{CLM}} := [\delta^2 y(t), \delta y(t), \textup{sign} \big( \delta y(t) \big), y(t) ]^T$. 
The transformed inputs enter the NN, as this was shown to enhance convergence of the PGNN training for high sampling rates in \cite{Bolderman2021}. 
Moreover, the NN part of the PGNNs~\eqref{eq:PGNNsCLM} have a single hidden layer with $n_l = 16$ neurons that have $\tanh (\cdot)$ activation function. 

Figure~\ref{fig:RegularizationEffect} shows the effect of the regularization using different values for $\Lambda = \lambda I$ on the training convergence for both the data fit and the LIP fit~\eqref{eq:RegularizedCostFunction}. 
There is a trend observable where a smaller $\lambda$ allows the physics--guided parameters $\theta_{\textup{phy}}$ to deviate more from the $\hat{\theta}_{\textup{LIP}}$ in order to improve the data fit.

\begin{figure}
    \centering
    \includegraphics[width=1\linewidth]{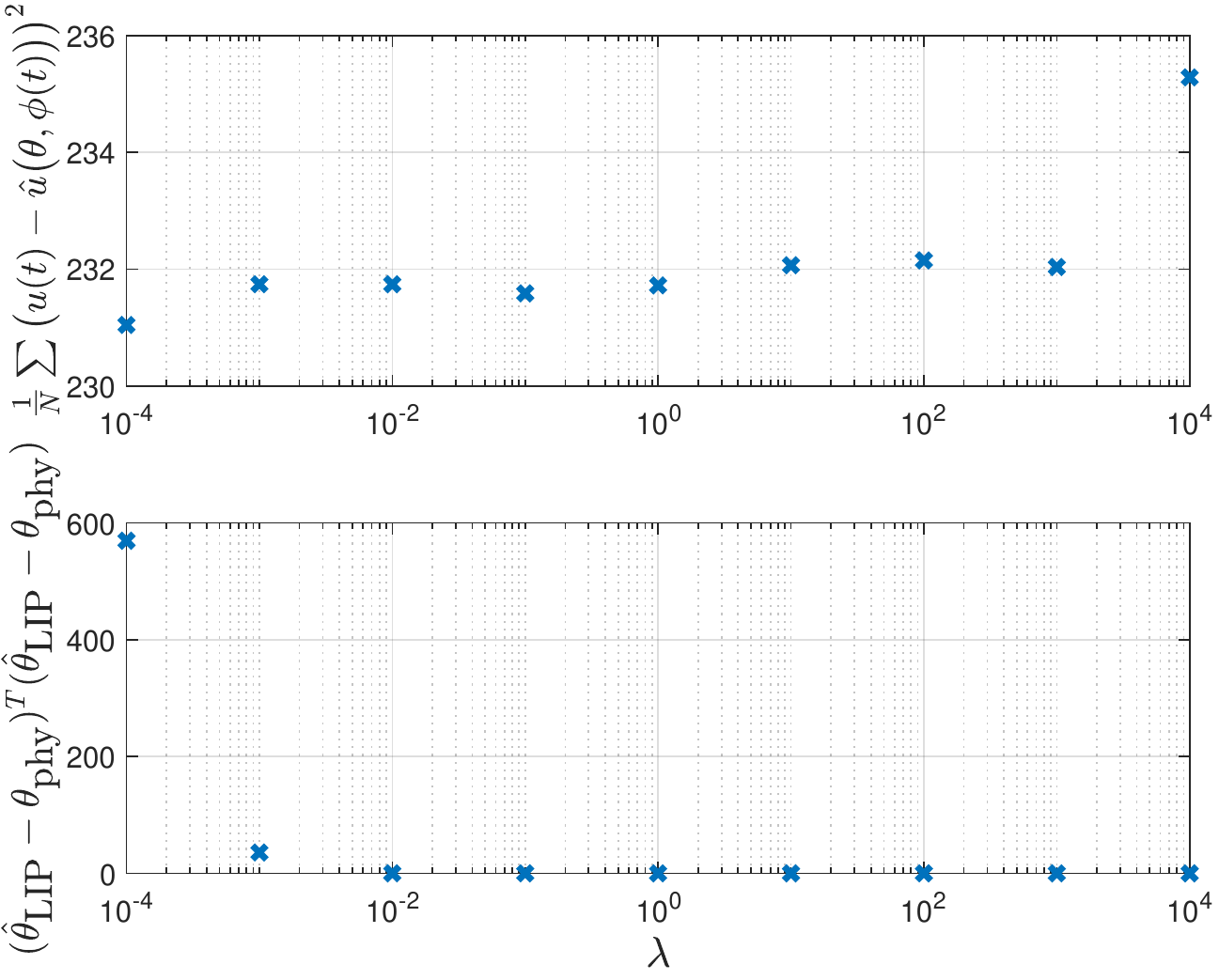}
    \caption{Effect of the regularization $\Lambda = \lambda I$~\eqref{eq:RegularizedCostFunction} on the training convergence for the data fit (top) and LIP fit (bottom). }
    \label{fig:RegularizationEffect}
\end{figure}

Figure \ref{fig:TrackingErrorA} shows the tracking error $e(t) := r(t)-y(t)$ resulting from the nominal reference $r_1(t)$ for the PGNNs~\eqref{eq:PGNN} that are trained either sequentially, or with regularized cost function~\eqref{eq:RegularizedCostFunction} using $\Lambda = \lambda I$ with $\lambda = 0$ or $\lambda = 0.01$. 
There is no major difference between the different PGNNs in terms of tracking performance, as is also confirmed by the data in Table~\ref{tab:Performance} that lists the mean--absolute error
\begin{equation}
    \label{eq:MAE}
    \textup{MAE} \big( e(t) \big) := \frac{1}{N_R} \sum_{t \in Z^R} \left| r(t) - y(t) \right|.
\end{equation}

Figure~\ref{fig:TrackingErrorB} shows the tracking error resulting from a reference $r_2(t)$ that oscillates between $r_2(t) \in \{ 0, 0.17 \}$ $m$, using the same bounds on velocity, acceleration, and jerk as $r_1(t)$.
The regularization helps to keep performance loss limited when operating the PGNN feedforwards on conditions not present in the training data.
This is also confirmed by the MAE shown in Table~\ref{tab:Performance}.
On the other hand, it is mainly due to the feedback controller that the MAE for $\lambda =0$ does not become larger on this experiment. 

Additionally, Table~\ref{tab:Performance} shows the results for $n_l = 8$ hidden layer neurons. 
For this situation, in which the NN part of the PGNN has limited approximation capabilities, it is clear that the regularized PGNN significantly outperforms the sequentially (consecutively) trained PGNN.
Currently, what limits the measurable tracking performance is the $0.5 \cdot 10^{-5}$ $m$ resolution of the linear encoder with which position measurements $y(t)$ are taken.

\begin{figure}
	\begin{subfigure}{1\linewidth}
	\centering
	\includegraphics[width=1\linewidth]{Figures/ReferenceA.pdf}
	\end{subfigure}\hfill
	\begin{subfigure}{1\linewidth}
	\centering
	\includegraphics[width=1\linewidth]{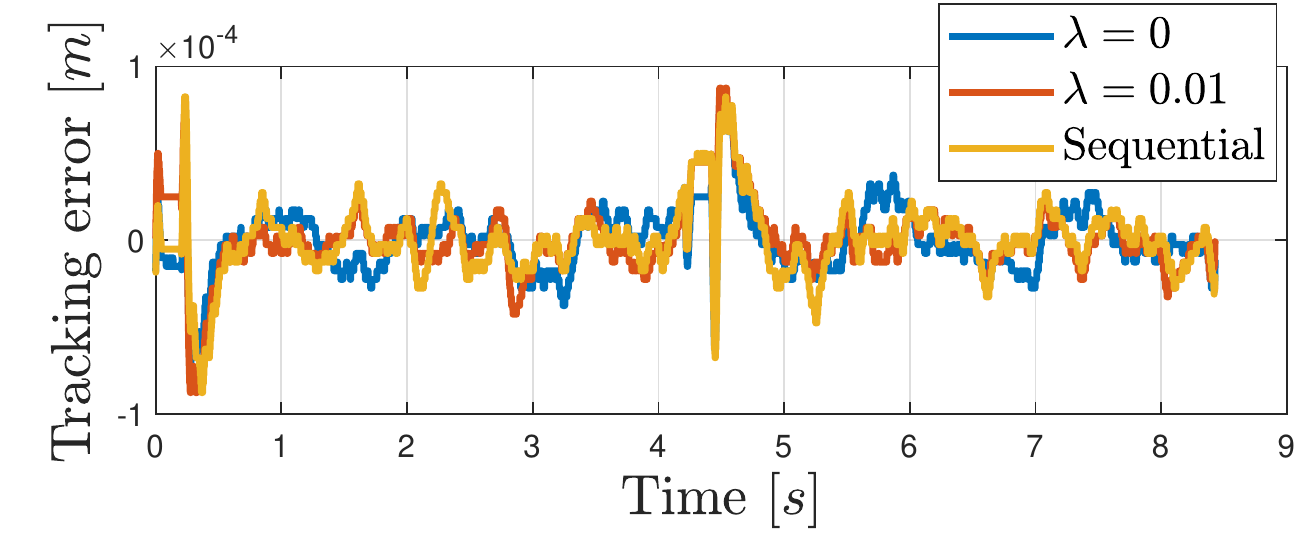}
	\end{subfigure}
	\caption{Tracking error on the nominal reference $r_1(t)$ used for training data generation.}
	\label{fig:TrackingErrorA}
\end{figure}

\begin{figure}
	\begin{subfigure}{1\linewidth}
	\centering
	\includegraphics[width=1\linewidth]{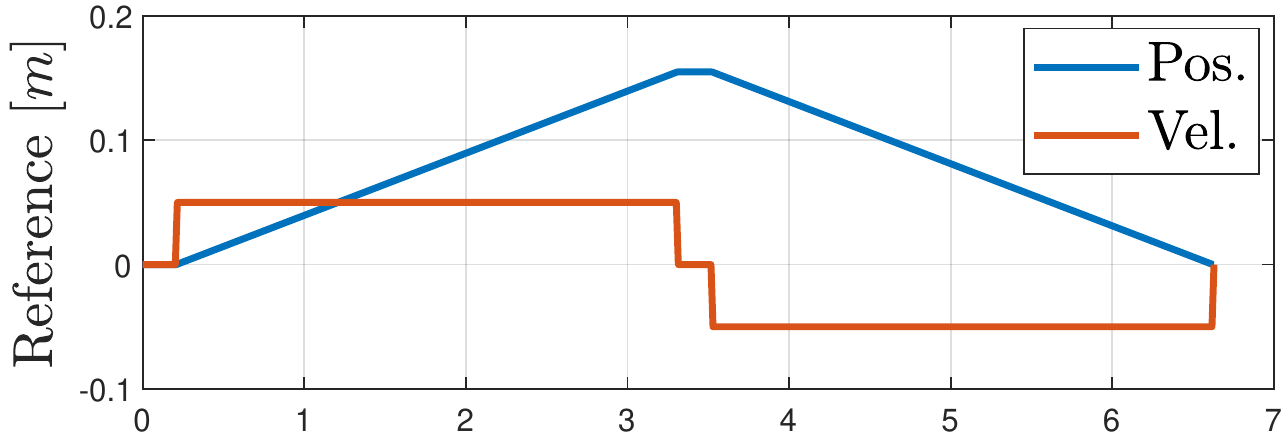}
	\end{subfigure}\hfill
	\begin{subfigure}{1\linewidth}
	\centering
	\includegraphics[width=1\linewidth]{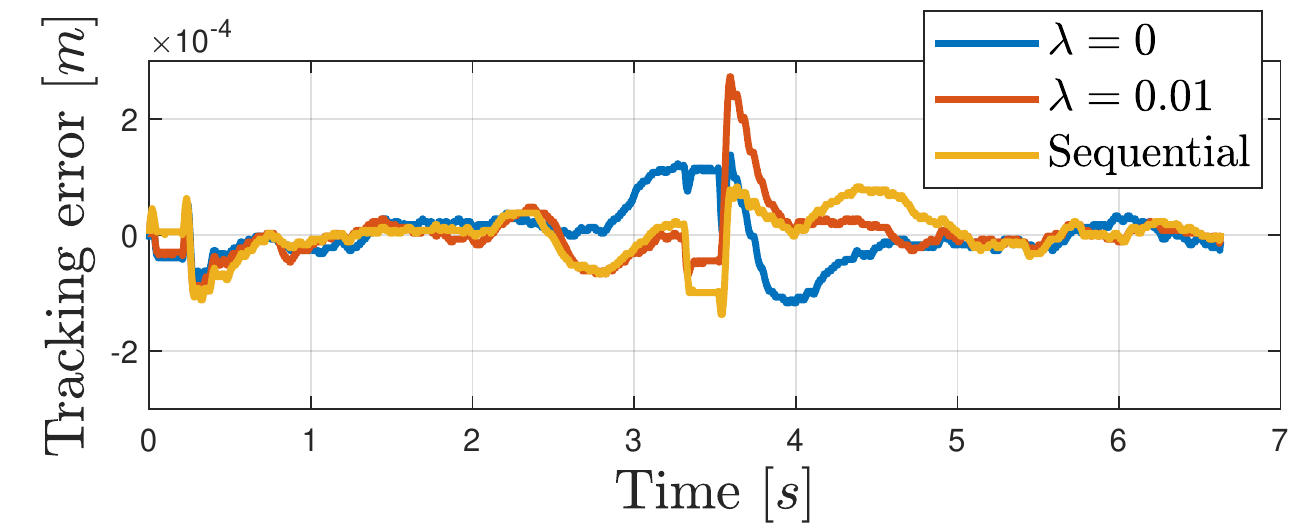}
	\end{subfigure}
	\caption{Tracking error for the reference $r_2(t)$ that attains positions not present in the training data.}
	\label{fig:TrackingErrorB}
\end{figure}

\begin{table}
\caption{MAE of the tracking error $e(t)$ for the nominal reference $r_1(t)$ and new reference $r_2(t)$. }
\label{tab:Performance}
\includegraphics[width=1.0\linewidth]{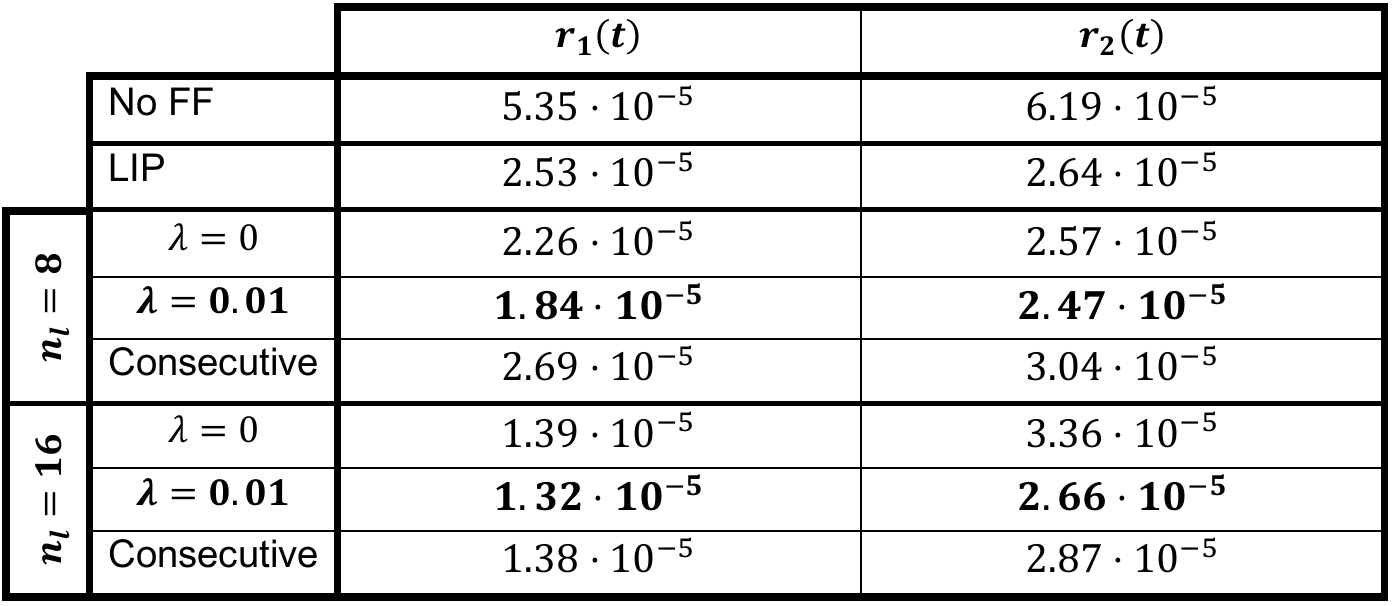}
\end{table}

\section{Conclusion}
\label{sec:Conclusion}
In this paper, a regularization--based PGNN feedforward control framework was introduced for enabling a trade--off between high data fitting accuracy of NNs and the good extrapolation properties of physical models.
An optimized initialization method was developed for the considered PGNN with a parellel structure, which has improved training convergence compared to a corresponding linear--in--the--parameter model. 
Experimental validation showed that developed regularization of the training cost is capable of optimizing the trade--off between data fit and physics--model based extrapolation.

\bibliographystyle{IEEEtran}
\bibliography{IEEEabrv,References}

\end{document}